\documentclass[letterpaper,10pt,conference]{ieeeconf}
\IEEEoverridecommandlockouts %
\usepackage{xcolor}
\usepackage{graphicx}
\graphicspath{{figs/}}

\usepackage{amsmath}
\usepackage{amssymb}
\usepackage{amsfonts}
\usepackage{amstext}

\usepackage{amsthm}
\usepackage{bm}
\usepackage{mathtools}
\allowdisplaybreaks

\usepackage{booktabs}

\theoremstyle{definition}
\newtheorem{definition}{Definition}
\newtheorem{theorem}{Theorem}

\usepackage[
  backend=biber,
  hyperref=true,
  sorting=none,
  style=ieee-caps]
{biblatex}

\usepackage[noabbrev]{cleveref}
\crefname{equation}{}{}

\usepackage{xparse}

\let\originalleft\left
\let\originalright\right
\renewcommand{\left}{\mathopen{}\mathclose\bgroup\originalleft}
\renewcommand{\right}{\aftergroup\egroup\originalright}

\NewDocumentCommand\Real{}{ \mathbb{R} }

\NewDocumentCommand\bbm{}{ \begin{bmatrix} } 
\NewDocumentCommand\ebm{}{ \end{bmatrix} }   
\NewDocumentCommand\T{}{\mathsf{T}}          

\NewDocumentCommand\Vector{m}{ \boldsymbol{\mathbf{#1}} }

\NewDocumentCommand\Matrix{m}{ \bm{\mathbf{#1}} }

\NewDocumentCommand\Norm{m}{ \left\Vert#1\right\Vert }

\NewDocumentCommand\Ker{m}{ \mathrm{ker}\left(#1\right) }

\NewDocumentCommand\Span{m}{ \mathrm{span}\left\{#1\right\} }

\NewDocumentCommand\Zero{}{ \Matrix{0} }
\NewDocumentCommand\Identity{}{ \Matrix{I} }

\NewDocumentCommand\LieGroupSO{m}{ \mathrm{SO}(#1) }
\NewDocumentCommand\LieAlgebraSO{m}{ \mathfrak{so}(#1) }

\NewDocumentCommand\LieGroupSGal{m}{ \mathrm{SGal}(#1) }
\NewDocumentCommand\LieAlgebraSGal{m}{ \mathfrak{sgal}(#1) }

\NewDocumentCommand\Matlog{m}{\mathrm{ln}\left(#1\right)}

\NewDocumentCommand\Matexp{m}{\exp\left(#1\right)}

\NewDocumentCommand\Orb{m}{\mathrm{Orb}\left(#1\right)}

\NewDocumentCommand\Defined{}{ \triangleq }

\usepackage{scalerel}

\NewDocumentCommand\<{}{\mspace{1mu}}

\newcommand{\uw}{\Vector{u}^{\wedge}}

\newcommand{\rms}[1]{\mathrm{#1}}

\NewDocumentCommand\StateVector{}{ \Vector{\mathcal{X}} }

\NewDocumentCommand\SmallStateVector{}{\scaleobj{0.85}{\StateVector}}

\title{\Large\bf Lost in Time? Continuous Symmetry and Identifiability in\\ Aided Inertial Navigation with Unknown Measurement Delays}

\author{Jonathan Kelly$^{1}$, Phone Thiha Kyaw$^{1}$, and Matthew Giamou$^{2}$
\thanks{$^{1}$Jonathan Kelly and Phone Thiha Kyaw are with the Space \& Terrestrial Autonomous Robotic Systems (STARS) Laboratory at the University of Toronto Institute for Aerospace Studies (UTIAS), Toronto, Canada. {\tt\footnotesize <first\_name>.<last\_name>@robotics.utias.utoronto.ca}}
\thanks{$^{2}$Matthew Giamou is with the Autonomous Robotics \& Convex Optimization (ARCO) Laboratory in the Department of Computing and Software, McMaster University, Hamilton, Canada. {\tt\footnotesize giamoum@mcmaster.ca}}}

\begin{document}
\maketitle

\begin{abstract}
In many multisensor systems, measurements from different sensors are subject to unknown relative time delays.
Accurate state estimation requires that delays be accounted for and, when possible, calibrated online.
We consider the case of aided inertial navigation, where measurements from a single aiding sensor are subject to an unknown but constant delay relative to the inertial measurement data stream, and study the identifiability of the resulting system.
Critically, identifiability depends not only on the temporal structure of the measurements, but also on the shape of the vehicle trajectory: some trajectories are sufficiently informative to support unique recovery of the delay and the navigation state, while others are not.
Using the special Galilean Lie group, we characterize a broad family of uninformative trajectories, each generated by a constant element of the Galilean Lie algebra.
We show that, along any such trajectory, the delayed measurement model admits a continuous symmetry that prevents unique recovery of the delay and the navigation state.
We connect this symmetry-based characterization to the familiar linearized, Jacobian-based analysis.
Although our development is motivated by aided navigation, the underlying ideas apply more generally to estimation problems on Lie groups with delayed measurements.
\end{abstract}

\section{Introduction}
\label{sec:introduction}

An aided inertial navigation system (INS) fuses inertial measurements with data from additional sensors to estimate the state of a moving vehicle, typically including its position, velocity, and orientation.
The inertial measurement unit (IMU) supplies high-rate angular velocity and specific-force readings that drive the process model; lower-rate aiding sensors, such as GNSS receivers or cameras, provide observations that correct drift to maintain bounded error over longer time horizons.

In practice, measurements from different sensors are subject to relative time offsets, or delays.
These delays arise from sensing, processing, and communication latency, and are often modelled as constant but unknown~\cite{2014_Kelly_Determining,2014_Li_Online,2018_Qin_Online}.
For optimal navigation performance, delays must be properly accounted for within the state estimator.
Failure to do so can lead to biased state estimates, inconsistent covariance estimates, and, in the worst case, divergence of the navigation solution.
The estimation problem for a single aiding sensor is therefore to recover the unknown delay, the navigation state, and possibly other parameters, from the IMU and delayed aiding sensor measurements.

One approach to online delay-and-state recovery augments the estimator state with the delay parameter and applies a standard nonlinear filtering algorithm, such as the extended Kalman filter~\cite{2010_Nilsson_Joint,2014_Li_Online,2021_Goudar_Online}.
This method is attractive because it fits easily within existing recursive estimation frameworks.
However, recovery is possible only if the augmented system is \emph{identifiable} (or observable, in the usual state-estimation sense), meaning that the delay and navigation state must be uniquely determined by the measurement history.
We previously~\cite{2021_Kelly_Question} showed that recursive filters suffer from structural issues that lead to biased and inconsistent estimates.
The present work addresses the question of whether the required information is present in the measurements at all, and explains the main issue with recursive updates: the system is unidentifiable from any single measurement because, as we show, the measurement model admits a symmetry under which distinct delay--state pairs are indistinguishable.

\begin{figure}[t!]
\centering
\vspace*{-4mm}
\includegraphics[width=0.985\columnwidth,trim={91.5px, 215px, 96px, 170px},clip]%
{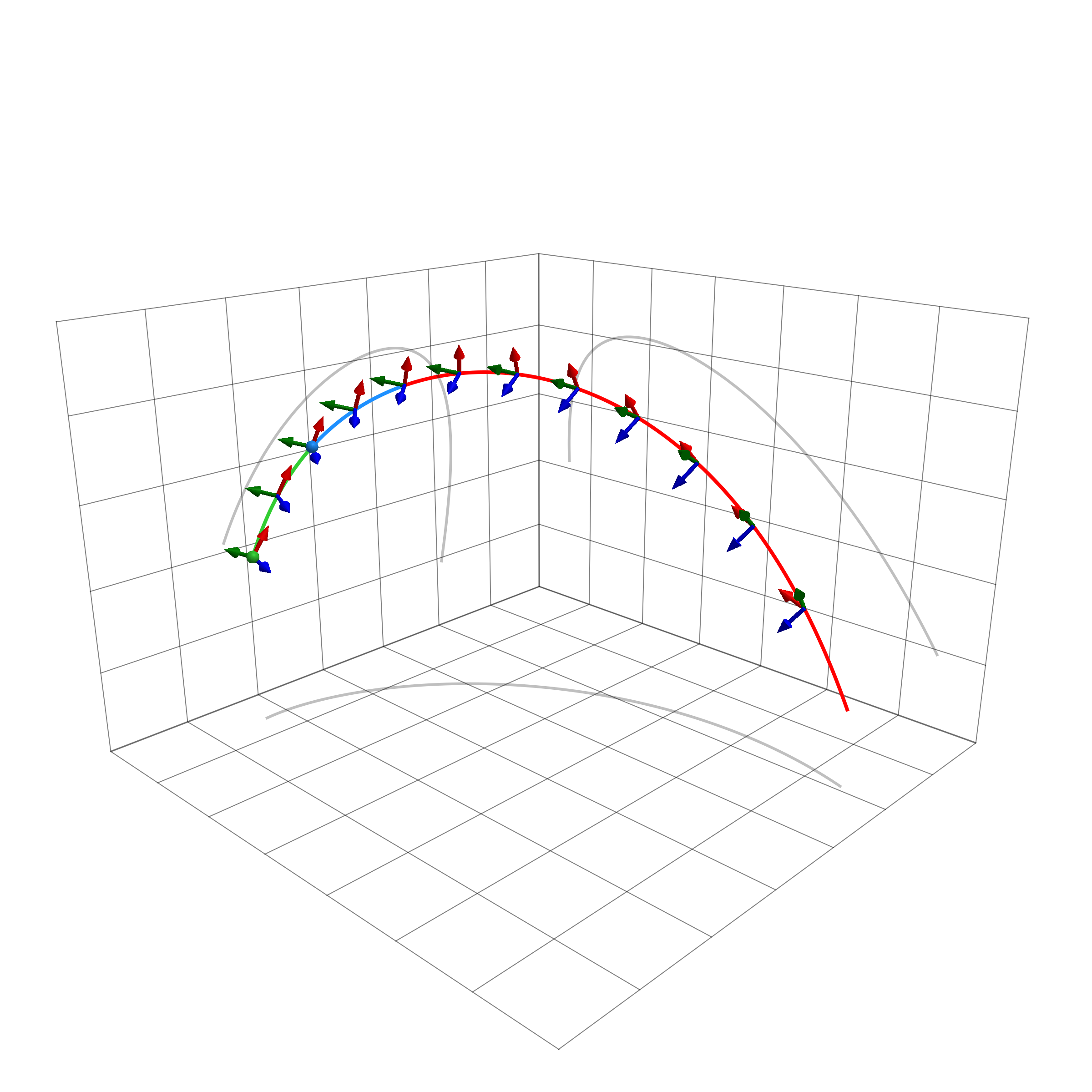}
\vspace*{-2mm}
\caption{An example uninformative trajectory for aided inertial navigation with delayed aiding measurements.
The green and blue spheres mark two distinct navigation states, both at the initial time $t = 0$, each paired with a different aiding-measurement delay; the corresponding trajectories are shown in green and blue, respectively. 
Both evolve under the same constant element of the Lie algebra $\LieAlgebraSGal{3}$ and trace the same curve under gravity.
Aiding measurements are available only along the red segment, and are consistent with either pairing, so the measurement history fails to uniquely determine the unknown delay and initial navigation state.
Body-frame axes are plotted as triads along the trajectory.
The grey curves are projections of the trajectory onto the coordinate planes and show its overall shape.}
\label{fig:trajectory_teaser}
\vspace*{-4.5mm}
\end{figure}

In this paper, we study identifiability in aided inertial navigation under measurement delays, and ask when the vehicle motion itself prevents unique recovery of the delay and the navigation state.
We call motions for which unique recovery fails \emph{uninformative trajectories}.
Along uninformative trajectories, joint online estimation may become ill-conditioned, with consequences ranging from reduced accuracy to estimator instability. %

Our analysis relies on the special Galilean group $\LieGroupSGal{3}$.
This $10$-dimensional Lie group is a natural state space for inertial motion, combining position, velocity, orientation, and time in a single geometric structure.
Its Lie algebra contains the space of IMU-driven inputs, that is, angular velocity and specific force.
This structure allows delay identifiability to be examined directly in terms of the vehicle trajectory and the delayed measurement history.

The central contribution of the paper is a symmetry-based characterization of identifiability in aided inertial navigation.
We show that the family of uninformative trajectories is determined by a continuous symmetry, induced by Galilean transformations, that preserves the measurement history.
Because the symmetry is geometric, the resulting characterization gives a clearer picture of why identifiability fails than a linearized, Jacobian-based analysis alone.
Although our presentation is motivated by aided inertial navigation, the same ideas apply more broadly to kinematic systems in which delayed measurements are used to correct a propagated state.

The remainder of the paper is organized as follows.
\Cref{sec:background} reviews related work, \Cref{sec:identifiability_symmetry} introduces identifiability and continuous symmetry concepts, and \Cref{sec:SGal3_group} develops the special Galilean group.
\Cref{sec:aided_identifiability} formulates delayed aided inertial navigation on $\LieGroupSGal{3}$ and analyzes local identifiability, while \Cref{sec:uninformative_trajectories} presents our characterization of uninformative trajectories.
Finally, \Cref{sec:conclusion} summarizes and suggests directions for future work.

\section{Background and Related Work}
\label{sec:background}

Structural identifiability asks whether unknown parameters can be recovered from noiseless output data~\cite{1970_Bellman_Structural,1976_Grewal_Identifiability}.
Observability is closely related, but concerns a time-varying state rather than fixed parameters and, for nonlinear delay-free systems, is assessed by rank tests involving Lie derivatives or observability Gramians~\cite{1977_Hermann_Nonlinear,2021_Grebe_Observability-Aware}.
The two viewpoints coincide in our setting: as we explain in \Cref{sec:aided_identifiability}, retaining the IMU input history makes the trajectory a function of fixed unknowns, namely the initial navigation state, the delay, and any additional parameters such as biases.%
\footnote{We refer to the initial navigation state as the \emph{initial condition} in the sequel, to match the notation of \Cref{eqn:general_delayed_system}.}
Filters, smoothers, and sliding-window optimizers all use this parameterization implicitly. We use the language of identifiability throughout.

Our approach is to determine when the measurement model admits a nontrivial continuous symmetry: a transformation of the unknown parameters that leaves the outputs unchanged.
Martinelli~\cite{2011_Martinelli_State,2020_Martinelli_Observability} developed this idea for delay-free systems using Lie derivatives, leading to symmetries characterized by local differential conditions.
In our case, the symmetries act on the entire measurement history, and whether they are trivial depends on the shape of the trajectory over time.
Khosravian et al.~\cite{2016_Khosravian_State} designed observers for invariant systems on Lie groups with delayed output measurements, but assumed the delay to be known. 
Here we ask whether an unknown delay can be determined.
Yang et al.~\cite{2019_Yang_Degenerate,2023_Yang_Online} identified specific uninformative trajectories for monocular camera--IMU systems through a Jacobian-based analysis.\footnote{Such trajectories are sometimes called \emph{degenerate} in the literature~\cite{2019_Yang_Degenerate,2023_Yang_Online}. We use the term uninformative to emphasize that the issue is insufficient excitation, rather than any defect in the trajectory itself.}
We consider aiding sensors that provide delayed global navigation measurements, which leads to a different family of trajectories, as described in \Cref{sec:uninformative_trajectories}.

Although we do not examine estimator performance here, uninformative trajectories are of practical importance.
In a related observability setting, Huang~\cite{2017_Huang_Towards} showed that apparent information may be gained along unobservable (unidentifiable) directions in the state (parameter) space, leading to estimator inconsistency and possibly to divergence.

\section{Identifiability and Symmetry}
\label{sec:identifiability_symmetry}

In this section, we introduce the identifiability framework used throughout the paper and connect it to symmetry.
Consider a nonlinear system of the form
\begin{equation}
\label{eqn:general_delayed_system}
\begin{gathered}
\dot{\Matrix{X}}(t)
=
f\bigl(
\Matrix{X}(t), \Vector{u}(t), \Vector{\theta}
\bigr),
\quad
\Matrix{X}(0) = \Matrix{X}_0 \in \mathcal{M}, \\[1mm]
\Matrix{Y}(t)
=
h\bigl(\Matrix{X}(t - \tau),\,\tau\bigr),
\quad
t \geq \tau,
\end{gathered}
\end{equation}
where $\mathcal{M}$ is an $m$-dimensional smooth manifold, $\Matrix{X}(t) \in \mathcal{M}$ is the state, $\Vector{u}(t) \in \mathcal{U} \subseteq \Real^p$ is a known control input, and $\Vector{\theta} \in \Theta \subseteq \Real^q$ is a vector of additional unknown parameters.
The map $f : \mathcal{M} \times \mathcal{U} \times \Theta \to T\mathcal{M}$ is a smooth vector field, with $f(\Matrix{X},\Vector{u},\Vector{\theta}) \in T_{\Matrix{X}}\mathcal{M}$.
The output $\Matrix{Y}(t) \in \mathcal{N}$ lies on an $\ell$-dimensional smooth manifold $\mathcal{N}$, the map $h:\mathcal{M} \times \Real_{> 0}\to\mathcal{N}$ is smooth, and $\tau \in (0, \tau_{\rms{u}})$ is an unknown constant delay, where $\tau_{\rms{u}}$ is a known upper bound.

We collect all the unknown parameters into a tuple,
\begin{equation*}
\StateVector
\Defined
\left(\Matrix{X}_0,\,\tau,\,\Vector{\theta}\right)
\in \mathcal{P},
\end{equation*}
where $\Matrix{X}_0 = \Matrix{X}(0)$ is the initial condition and $\mathcal{P}$ is the admissible parameter space.
We use $\SmallStateVector \in \Real^d$ for a local-coordinate vector corresponding to $\StateVector$, with $d = \dim(\mathcal{P})$.

We assume that the trajectory is defined for all $t \geq 0$, but that measurements are available only over a finite observation interval.
For identifiability, output histories must be compared over an interval that is valid for every candidate delay.
Together with the choice of observation interval below, the bound $\tau_{\rms{u}}$ ensures that the delayed state is well defined for every candidate delay, not only for the true delay $\tau$.

\subsection{Distinguishability, Identifiability, and Observability}
\label{subsec:identifiability}

We take the standard approach to determining identifiability based on output distinguishability~\cite{1970_Bellman_Structural,1976_Grewal_Identifiability}.
The definitions below are adapted to our delayed-measurement setting.
As is standard in structural identifiability analysis, we treat the system in \Cref{eqn:general_delayed_system} as noise-free.

Let $[t_{\rms{s}},\,t_{\rms{s}} + T]$ be the observation interval, with $t_{\rms{s}} \geq \tau_{\rms{u}}$.
For a known input history sufficient to propagate the state from the initial time to $t_{\rms{s}} + T$, each parameter tuple $\StateVector \in \mathcal{P}$ determines an output history over $[t_{\rms{s}},\, t_{\rms{s}} + T]$.
Changing the delay $\tau$ shifts which points along the trajectory enter the measurement model, making delay identifiability a trajectory-dependent question.

\begin{definition}[Indistinguishability]
Let $\mathcal{T} \subseteq [t_{\rms{s}},\, t_{\rms{s}} + T]$ be a set of observation times.
The parameter tuples $\StateVector^\circ, \StateVector^\star \in \mathcal{P}$ are said to be \emph{indistinguishable} on $\mathcal{T}$ if, for the same known input history, they produce identical outputs at every time in $\mathcal{T}$.
Otherwise, the parameter tuples are said to be \emph{distinguishable} on $\mathcal{T}$.
\end{definition}

Note that the corresponding delayed states may occur at different times along the trajectory, since the delays in $\StateVector^\circ$ and $\StateVector^\star$ need not be equal.

\begin{definition}[Identifiability]
A parameter tuple $\StateVector^\star \in \mathcal{P}$ is \emph{identifiable} on a set of observation times $\mathcal{T} \subseteq [t_{\rms{s}},\, t_{\rms{s}} + T]$ if it is distinguishable on $\mathcal{T}$ from every $\StateVector \in \mathcal{P}$ with $\StateVector \neq \StateVector^\star$.
\end{definition}

The preceding definition is global: it requires uniqueness over the entire admissible parameter space $\mathcal{P}$.
For nonlinear systems, this can be difficult to establish and is often stronger than necessary.
We therefore use a local notion, which requires uniqueness only within a neighbourhood of the parameters under consideration.

\begin{definition}[Local identifiability]
A parameter tuple $\StateVector^\star \in \mathcal{P}$ is \emph{locally identifiable} on a set of observation times $\mathcal{T} \subseteq [t_{\rms{s}},\, t_{\rms{s}} + T]$ if there exists a neighbourhood $\mathcal{V}$ of $\StateVector^\star$ such that $\StateVector^\star$ is distinguishable on $\mathcal{T}$ from every $\StateVector \in \mathcal{V}$ with $\StateVector \neq \StateVector^\star$.
Equivalently, any $\StateVector \in \mathcal{V}$ that produces the same outputs as $\StateVector^\star$ at every time in $\mathcal{T}$ must satisfy $\StateVector = \StateVector^\star$.
\end{definition}

Local identifiability can be assessed by linearizing the measurement model with respect to the unknown parameters.
Each measurement contributes a Jacobian block, and these blocks are stacked to form the complete measurement Jacobian.
Full column rank of the stacked Jacobian implies that the linearized measurements locally constrain the unknown parameters to first order, while rank deficiency indicates first-order perturbation directions that remain unconstrained.

\subsection{Continuous Symmetries}
\label{subsec:symmetries}

We now state the symmetry principle used later in the paper.
Our interest is in a continuous family of transformations of the unknown parameters that leaves the delayed output history unchanged over the observation interval.

\begin{theorem}[Symmetry implies unidentifiability]
\label{thm:symmetry_unidentifiable_continuous}
Let $\StateVector^\star \in \mathcal{P}$.
Suppose there exist $\epsilon > 0$ and a one-parameter family of transformations $\mathcal{S}_\alpha : \mathcal{P} \to \mathcal{P}$, $\alpha \in (-\epsilon, \epsilon)$, smooth in $\alpha$, such that $\mathcal{S}_0(\StateVector)=\StateVector$ for all $\StateVector$ in a neighbourhood of $\StateVector^\star$.
Suppose further that, for every $\alpha \in (-\epsilon, \epsilon)$, the parameter tuples $\StateVector^\star$ and $\mathcal{S}_\alpha(\StateVector^\star)$ are indistinguishable on $[t_{\rms{s}},\, t_{\rms{s}} + T]$.
If $\mathcal{S}_\alpha(\StateVector^\star) \neq \StateVector^\star$ for all $\alpha \in (-\epsilon, \epsilon) \setminus \{0\}$, then $\StateVector^\star$ is not locally identifiable on $[t_{\rms{s}},\, t_{\rms{s}} + T]$.
\end{theorem}

\begin{proof}
By continuity, $\mathcal{S}_\alpha(\StateVector^\star) \to \StateVector^\star$ as $\alpha \to 0$, so every neighbourhood of $\StateVector^\star$ contains some $\mathcal{S}_\alpha(\StateVector^\star) \neq \StateVector^\star$ that is indistinguishable from $\StateVector^\star$ on $[t_{\rms{s}},\, t_{\rms{s}} + T]$.
Hence the parameter-to-output map is not locally injective at $\StateVector^\star$, and $\StateVector^\star$ is not locally identifiable.
\end{proof}

In typical aided navigation systems, however, measurements are available only at discrete times.
The same argument applies directly to this case: if a nontrivial smooth family
$\mathcal{S}_\alpha$ leaves the outputs unchanged at the measurement times
$t_1, \dots, t_n$, then the sampled parameter-to-output map is not locally
injective.
Thus, $\StateVector^\star$ is not locally identifiable from the discrete
measurements.

The symmetries $\{\mathcal{S}_\alpha\}$ considered here form a smooth one-parameter family acting on the parameter space, with $\mathcal{S}_0$ the identity and $\mathcal{S}_\alpha \circ \mathcal{S}_\beta = \mathcal{S}_{\alpha+\beta}$ whenever both sides are defined.
As the parameter $\alpha$ varies, the point $\mathcal{S}_\alpha(\StateVector^\star)$ moves through parameter space while leaving the outputs unchanged, and the resulting set is the symmetry orbit through $\StateVector^\star$.

\begin{definition}[Symmetry orbit]
\label{def:symmetry_orbit}
The \emph{symmetry orbit} of $\StateVector^\star \in \mathcal{P}$ under the symmetry family $\{\mathcal{S}_\alpha\}$ is
\begin{equation*}
\Orb{\StateVector^\star}
=
\bigl\{\,\mathcal{S}_\alpha(\StateVector^\star) :
\alpha \in (-\epsilon,\epsilon)\,\bigr\}.
\end{equation*}
\end{definition}

Every point of $\Orb{\StateVector^\star}$ produces the same outputs as $\StateVector^\star$ over the observation interval, or at the specified measurement times in the discrete case.
If every neighbourhood of $\StateVector^\star$ contains points of $\Orb{\StateVector^\star}$ other than $\StateVector^\star$ itself, then $\StateVector^\star$ is not locally identifiable, by \Cref{thm:symmetry_unidentifiable_continuous}.

\section{The Special Galilean Group}
\label{sec:SGal3_group}

The special Galilean group $\LieGroupSGal{3}$ is a 10-dimensional Lie group that describes transformations between inertial reference frames in relative motion~\cite{2023_Kelly_Galilean,2025_Mahony_Galilean}.
In matrix form,
\begin{equation*}
\LieGroupSGal{3} \Defined
\left\{
\Matrix{X} =
\bbm
\Matrix{C} & \Vector{v} & \Vector{r} \\
\Zero & 1 & \eta \\
\Zero & 0 & 1
\ebm
\in \Real^{5 \times 5}
\ \middle\vert \
\Matrix{C} \in \LieGroupSO{3}
\right\},
\end{equation*}
where $\Matrix{C}$ is a rotation matrix, $\Vector{v} \in \Real^3$ is a velocity boost, $\Vector{r} \in \Real^3$ is a spatial translation, and $\eta \in \Real$ is a time translation.
The group operation is matrix multiplication.

A group element acts on spacetime coordinates $(\Vector{p}, t)$, with $\Vector{p} \in \Real^3$, according to
\begin{equation*}
(\Vector{p},t)
\mapsto
(\Matrix{C}\<\Vector{p} + \Vector{v}\<t + \Vector{r},\,t + \eta).
\end{equation*}
In this sense, $\LieGroupSGal{3}$ packages the kinematics of inertial motion into a single matrix group.

\subsection{The Lie Algebra $\LieAlgebraSGal{3}$}
\label{subsec:SGal3_lie_algebra}

The Lie algebra $\LieAlgebraSGal{3}$ is the tangent space of $\LieGroupSGal{3}$ at the identity.
In matrix form,
\begin{equation*}
\LieAlgebraSGal{3}
\Defined
\left\{
\Matrix{\Xi} =
\bbm
\Vector{\phi}^{\wedge} & \Vector{\nu} & \Vector{\rho} \\
\Zero & 0 & \iota \\
\Zero & 0 & 0
\ebm
\ \middle\vert \
\Vector{\phi},\, \Vector{\nu},\, \Vector{\rho} \in \Real^3,\,
\iota \in \Real
\right\},
\end{equation*}
where $\Vector{\phi}$ parameterizes the rotational part, $\Vector{\nu}$ the boost, $\Vector{\rho}$ the spatial translation, and $\iota$ the temporal translation.
Here, $(\cdot)^{\wedge}$ applied to $\Vector{\phi} \in \Real^3$ produces a skew-symmetric matrix, giving the standard isomorphism between $\Real^3$ and $\LieAlgebraSO{3}$.
The same notation is overloaded for $\LieAlgebraSGal{3}$ via the map $(\cdot)^{\wedge} : \Real^{10} \to \LieAlgebraSGal{3}$ defined as
\begin{equation}
\label{eqn:sgal3_wedge}
\Vector{\xi}^{\wedge}
=
\bbm
\Vector{\rho} \\
\Vector{\nu} \\
\Vector{\phi} \\
\iota
\ebm^{\wedge}
=
\bbm
\Vector{\phi}^{\wedge} & \Vector{\nu} & \Vector{\rho} \\
\Zero & 0 & \iota \\
\Zero & 0 & 0
\ebm.
\end{equation}
The corresponding operator $(\cdot)^{\vee}:\LieAlgebraSGal{3} \to \Real^{10}$ is the inverse of $(\cdot)^{\wedge}$, so that if $\Vector{\xi}^{\wedge} = \Matrix{\Xi}$, then $\Matrix{\Xi}^{\vee} = \Vector{\xi}$.

\subsection{The Exponential Map}
\label{subsec:SGal3_exp_map}

Let $\Vector{\xi} = \bbm \Vector{\rho}^\T\;\; \Vector{\nu}^\T\;\; \Vector{\phi}^\T\;\; \iota \ebm^\T \in \Real^{10}$.
The exponential map is defined by the matrix exponential and, for $\LieGroupSGal{3}$, has the closed form
\begin{equation}
\label{eqn:SGal3_exp_closed}
\Matexp{\Vector{\xi}^{\wedge}}
=
\bbm
\Matrix{C} & \Matrix{D}\<\Vector{\nu} & \Matrix{D}\<\Vector{\rho} + \Matrix{E}\<\Vector{\nu}\iota \\
\Zero & 1 & \iota \\
\Zero & 0 & 1
\ebm,
\end{equation}
where the three $3 \times 3$ matrices $\Matrix{C}$, $\Matrix{D}$, and $\Matrix{E}$ are
\begin{equation}
\label{eqn:C_D_E_matrices}
\begin{aligned}
\Matrix{C}
& =
\Identity_{3}
+
\sin(\phi)\,\uw
+
\bigl(1-\cos(\phi)\bigr)\left.\uw\right.^{\!2}, \\[1mm]
\Matrix{D}
& =
\Identity_{3}
+
\left(\frac{1-\cos(\phi)}{\phi}\right)\uw
+
\left(\frac{\phi-\sin(\phi)}{\phi}\right)\left.\uw\right.^{\!2}, \\
\Matrix{E}
& =
\frac{\Identity_{3}}{2}
+
\left(\frac{\phi-\sin(\phi)}{\phi^{2}}\right)\uw
+
\left(\frac{\phi^{2} + 2\cos(\phi)-2}{2\phi^{2}}\right)\left.\uw\right.^{\!2},
\end{aligned}
\end{equation}
and $\Vector{\phi} = \phi\<\Vector{u}$, with $\phi = \Norm{\Vector{\phi}}$, $\Vector{u}$ a unit vector, and $\uw$ the corresponding skew-symmetric matrix.
The expressions above are evaluated using their Taylor-series limits as $\phi \to 0$.

\section{Aided Navigation and Identifiability}
\label{sec:aided_identifiability}

In this section, we develop a model for aided inertial navigation with delayed aiding measurements, making use of the special Galilean group.
The process model is driven by IMU body-frame angular-rate and specific-force measurements.
We treat the IMU biases as constant\footnote{In reality, the IMU gyroscope and accelerometer biases vary slowly with time. We treat them as constant because the measurement windows we consider are comparatively short.} and also assume a known constant gravity vector.
We initially study local identifiability using the traditional approach: we linearize the delayed measurement model and analyze the resulting Jacobian.
The analysis reveals first-order perturbation directions that the delayed measurements leave unconstrained, but it does not by itself explain their geometric origin.
In \Cref{sec:uninformative_trajectories}, we show that these directions arise from a continuous symmetry of the delayed measurement model, giving a geometric characterization of uninformative trajectories.

\subsection{Navigation State and Kinematics}
\label{subsec:aided_system}

Consider a Galilean transformation between a moving body frame attached to the vehicle and a fixed inertial frame.
To keep the notation light, we omit explicit frame labels.
The navigation state on $\LieGroupSGal{3}$ is
\begin{equation}
\label{eqn:SGal3_state}
\Matrix{X}(t)
=
\bbm
\Matrix{C}(t) & \Vector{v}(t) & \Vector{r}(t) \\
\Zero & 1 & t \\
\Zero & 0 & 1
\ebm,
\end{equation}
where $\Matrix{C}(t) \in \LieGroupSO{3}$ is the orientation of the body frame relative to the inertial frame, and $\Vector{v}(t), \Vector{r}(t) \in \Real^3$ are the velocity and position expressed in the inertial frame.
The scalar $t$ is the time coordinate of the Galilean transformation, defined relative to the inertial-frame time origin.
As we show below, keeping this coordinate in the group element lets us express output delays as time translations on the group.

The unknown parameters are
\begin{equation*}
\StateVector
=
\left(
\Matrix{X}_0,\,
\tau,\,
\Vector{b}_{\omega},\,
\Vector{b}_{\rms{a}}
\right),
\end{equation*}
where $\Matrix{X}_0$ is the initial condition, $\tau$ is the unknown delay, and $\Vector{b}_{\omega},\Vector{b}_{\rms{a}} \in \Real^3$ are the gyroscope and accelerometer biases, respectively.
Without loss of generality, we choose the initial condition to lie in the \emph{isochronous} Galilean subgroup~\cite{2023_Kelly_Galilean}, consistent with the initialization at $t = 0$ in \Cref{eqn:general_delayed_system}.
This means that the time coordinate of $\Matrix{X}_0$ is zero.
Using this convention,
\begin{equation*}
\Matrix{X}_0
=
\bbm
\Matrix{C}_0 & \Vector{v}_0 & \Vector{r}_0 \\
\Zero & 1 & 0 \\
\Zero & 0 & 1
\ebm
\in \LieGroupSGal{3},
\end{equation*}
where $\Matrix{C}_0 = \Matrix{C}(0)$, $\Vector{v}_0 = \Vector{v}(0)$, and $\Vector{r}_0 = \Vector{r}(0)$ are unknown.

Define the bias-corrected angular rate and specific force,
\begin{equation*}
\Vector{\omega}(t) \Defined \Vector{\omega}_{\rms{m}}(t) - \Vector{b}_{\omega},
\quad
\Vector{s}(t) \Defined \Vector{a}_{\rms{m}}(t) - \Vector{b}_{\rms{a}},
\end{equation*}
where $\Vector{\omega}_{\rms{m}}(t)$ and $\Vector{a}_{\rms{m}}(t) \in \Real^3$ are the \emph{measured} body-frame angular rate and specific force, respectively.
The continuous-time kinematics are
\begin{align*}
\dot{\Matrix{C}}(t)
& =
\Matrix{C}(t)\<\Vector{\omega}(t)^{\wedge}, \\
\dot{\Vector{v}}(t)
& =
\Matrix{C}(t)\<\Vector{s}(t) + \Vector{g}, \\
\dot{\Vector{r}}(t)
& =
\Vector{v}(t),
\end{align*}
where $\Vector{g} \in \Real^3$ is the inertial-frame gravitational acceleration.
The remaining time coordinate satisfies $\dot{t} = 1$.

We decompose the trajectory into contributions from the initial condition, the integrated input history, and gravity.
Let $\breve{\Matrix{C}}(t)$, $\breve{\Vector{v}}(t)$, and $\breve{\Vector{r}}(t)$ satisfy
\begin{align*}
\dot{\breve{\Matrix{C}}}(t)
& =
\breve{\Matrix{C}}(t)\<\Vector{\omega}(t)^{\wedge}, \\
\dot{\breve{\Vector{v}}}(t)
& =
\breve{\Matrix{C}}(t)\<\Vector{s}(t), \\
\dot{\breve{\Vector{r}}}(t)
& =
\breve{\Vector{v}}(t),
\end{align*}
with $\breve{\Matrix{C}}(0) = \Identity_{3}$ and $\breve{\Vector{v}}(0) = \breve{\Vector{r}}(0) = \Zero$, so that these quantities depend only on the IMU inputs and the biases.
Integrating the kinematics forward from $\Matrix{X}_0$ gives
\begin{equation}
\label{eqn:SGal3_trajectory_integrated}
\begin{aligned}
\Matrix{C}(t)
& =
\Matrix{C}_0\<\breve{\Matrix{C}}(t), \\
\Vector{v}(t)
& =
\Matrix{C}_0\<\breve{\Vector{v}}(t) + \Vector{v}_0 + t\<\Vector{g}, \\
\Vector{r}(t)
& =
\Matrix{C}_0\<\breve{\Vector{r}}(t) + t\<\Vector{v}_0 + \Vector{r}_0 + \tfrac{1}{2}\<t^2\<\Vector{g}.
\end{aligned}
\end{equation}

\subsection{Group Form of the Trajectory}
\label{subsec:group_trajectory}

The componentwise expressions above admit a compact form on $\LieGroupSGal{3}$.
Working directly on the group lets us define the trajectory and measurement models cleanly, and express the symmetry of \Cref{sec:uninformative_trajectories} exactly.
The three contributions to the trajectory become group factors: gravity acts on the left, in the inertial frame, while the integrated input history acts on the right, in the body frame, and the unknown initial condition sits between them.
We collect the propagated quantities into the state transition matrix
\begin{equation*}
\Matrix{\Phi}(t)
\Defined
\bbm
\breve{\Matrix{C}}(t) & \breve{\Vector{v}}(t) & \breve{\Vector{r}}(t) \\
\Zero & 1 & t \\
\Zero & 0 & 1
\ebm
\in \LieGroupSGal{3},
\quad
\Matrix{\Phi}(0) = \Identity_5.
\end{equation*}

The gravity factor is built from two Lie algebra directions.
Because time is part of the Galilean group element, a time shift is generated by
\begin{equation}
\label{eqn:pure_time_generator}
\Vector{\xi}_{\rms{d}}
\Defined
\bbm
\Zero_{1 \times 9} & 1
\ebm^\T
\in \Real^{10},
\end{equation}
which, with our ordering of the Lie algebra coordinates, has only a temporal component.
Write $\Matrix{T}(\beta) \Defined \Matexp{\beta\<\Vector{\xi}_{\rms{d}}^{\wedge}}$ for the corresponding pure time translation, which shifts the group time coordinate by $\beta$ and leaves position, velocity, and orientation unchanged.
Gravity is generated by
\begin{equation}
\label{eqn:sgal3_gravity_generator}
\Vector{\xi}_{\rms{g}}
\Defined
\bbm
\Zero_{1 \times 3} & \Vector{g}^\T & \Zero_{1 \times 3} & 1
\ebm^\T
\in \Real^{10},
\end{equation}
with $\Vector{\nu} = \Vector{g}$ and $\iota = 1$.
The associated one-parameter subgroup is
\begin{equation}
\label{eqn:SGal3_free_fall_subgroup}
\Matrix{W}(\beta)
\Defined
\Matexp{\beta\<\Vector{\xi}_{\rms{g}}^{\wedge}}
=
\bbm
\Identity_{3} & \beta\<\Vector{g} & \tfrac{1}{2}\<\beta^2\Vector{g} \\
\Zero & 1 & \beta \\
\Zero & 0 & 1
\ebm,
\end{equation}
where the closed form follows from \Cref{eqn:SGal3_exp_closed} with $\Vector{\phi} = \Zero$.
This subgroup describes free fall: over the interval $\beta$, the velocity changes by $\beta\<\Vector{g}$ and the position by $\tfrac{1}{2}\<\beta^2\Vector{g}$, while time advances by $\beta$.
Since $\Matrix{\Phi}(t)$ already advances the time coordinate, we remove the time translation from $\Matrix{W}(t)$, leaving the contribution of gravity over $[0,\,t]$,
\begin{equation}
\label{eqn:SGal3_gravity_factor}
\Matrix{G}(t)
\Defined
\Matrix{W}(t)\<\Matrix{T}(-t)
=
\bbm
\Identity_{3} & t\<\Vector{g} & -\tfrac{1}{2}\<t^2\Vector{g} \\
\Zero & 1 & 0 \\
\Zero & 0 & 1
\ebm.
\end{equation}
The trajectory then factors as
\begin{equation}
\label{eqn:SGal3_trajectory_factored}
\Matrix{X}(t)
=
\Matrix{G}(t)\<\Matrix{X}_0\<\Matrix{\Phi}(t),
\end{equation}
which reproduces \Cref{eqn:SGal3_trajectory_integrated} exactly.\footnote{A similar gravity factor appears in equivariant IMU preintegration on the special Galilean group~\cite{2025_Delama_Equivariant}.}
The negative sign of the position term in \Cref{eqn:SGal3_gravity_factor} is accounted for by the group product: since $\Matrix{G}(t)$ and $\Matrix{X}_0$ have zero time coordinate, the boost--time coupling contributes $t^2\Vector{g}$, which combines with $-\tfrac{1}{2}\<t^2\Vector{g}$ to give the expected $\tfrac{1}{2}\<t^2\Vector{g}$.

\subsection{Delayed Aiding Measurements}
\label{subsec:delayed_measurements}

We consider aiding measurements collected at discrete times $t_k$, for $k = 1, \dots, n$.
A measurement received at time $t_k$ corresponds to the navigation state at the earlier time $t_k - \tau$.
For compactness, we write $t_{\rms{d},k} \Defined t_k - \tau$ for the delayed measurement time. %

We assume that the aiding sensor measures the full delayed navigation state, that is, the position, velocity, and orientation at time $t_{\rms{d},k}$, reported at the output time $t_k$,
\begin{equation}
\label{eqn:SGal3_delayed_measurement}
\Matrix{Y}_k
\Defined
\Matrix{T}(\tau)\<
\Matrix{X}(t_{\rms{d},k})
=
\Matrix{T}(\tau)\<
\Matrix{G}(t_{\rms{d},k})\<
\Matrix{X}_0\<
\Matrix{\Phi}(t_{\rms{d},k}),
\end{equation}
where the leading factor $\Matrix{T}(\tau)$, defined in \Cref{subsec:group_trajectory}, advances the group time coordinate from $t_{\rms{d},k}$ to $t_k$.
Writing the navigation components explicitly,
\begin{align*}
\Matrix{C}(t_{\rms{d},k})
& =
\Matrix{C}_0\<\breve{\Matrix{C}}(t_{\rms{d},k}), \\
\Vector{v}(t_{\rms{d},k})
& =
\Matrix{C}_0\<\breve{\Vector{v}}(t_{\rms{d},k}) + \Vector{v}_0 + t_{\rms{d},k}\<\Vector{g}, \\
\Vector{r}(t_{\rms{d},k})
& =
\Matrix{C}_0\<\breve{\Vector{r}}(t_{\rms{d},k}) + t_{\rms{d},k}\<\Vector{v}_0 + \Vector{r}_0 + \tfrac{1}{2}\<t_{\rms{d},k}^2\<\Vector{g}.
\end{align*}

\begin{figure*}[t]
\centering
\begin{minipage}{0.95\textwidth}
\begin{equation}
\label{eqn:single_full_jacobian}
\Matrix{H}_k
=
\bbm
\Identity_{3}
&
t_{\rms{d},k}\Identity_{3}
&
-\Matrix{C}_0\<\breve{\Vector{r}}(t_{\rms{d},k})^{\wedge}
&
-\Vector{v}(t_{\rms{d},k})
&
\Matrix{C}_0\<\Matrix{\Psi}^{(r)}_{\omega}(t_{\rms{d},k})
&
\Matrix{C}_0\<\Matrix{\Psi}^{(r)}_{\rms{a}}(t_{\rms{d},k})
\\[0.5mm]
\Zero
&
\Identity_{3}
&
-\Matrix{C}_0\<\breve{\Vector{v}}(t_{\rms{d},k})^{\wedge}
&
-\Vector{a}(t_{\rms{d},k})
&
\Matrix{C}_0\<\Matrix{\Psi}^{(v)}_{\omega}(t_{\rms{d},k})
&
\Matrix{C}_0\<\Matrix{\Psi}^{(v)}_{\rms{a}}(t_{\rms{d},k})
\\[0.5mm]
\Zero
&
\Zero
&
\breve{\Matrix{C}}(t_{\rms{d},k})^\T
&
-\Vector{\omega}(t_{\rms{d},k})
&
\Matrix{\Psi}^{(\phi)}_{\omega}(t_{\rms{d},k})
&
\Zero
\ebm.
\end{equation}
\vspace{0.1\baselineskip}
\end{minipage}
\noindent\rule{\textwidth}{0.25pt}
\vspace{-2.3\baselineskip}
\end{figure*}

\subsection{Measurement Jacobian and Nullspace Structure}
\label{subsec:jacobian_nullspace}

Next, we examine local identifiability by linearizing the delayed measurement model.
We use additive perturbations of the inertial-frame position and velocity, together with a right-multiplicative orientation perturbation $\Matrix{C}_0 \mapsto \Matrix{C}_0\Matexp{\delta\boldsymbol{\phi}_0^{\wedge}}$, and collect the parameter perturbations into
\begin{equation*}
\delta\SmallStateVector
\Defined
\bbm
\delta\Vector{r}_0^\T &
\delta\Vector{v}_0^\T &
\delta\boldsymbol{\phi}_0^\T &
\delta\tau &
\delta\<\Vector{b}_{\omega}^\T &
\delta\<\Vector{b}_{\rms{a}}^\T
\ebm^\T
\in \Real^{16}
\end{equation*}
and the output perturbation into
\begin{equation*}
\delta\Vector{y}_k
\Defined
\bbm
\delta\Vector{r}_k^\T &
\delta\Vector{v}_k^\T &
\delta\boldsymbol{\vartheta}_k^\T
\ebm^\T
\in \Real^9,
\end{equation*}
where $\bar{\Matrix{Y}}_k$ denotes the measurement predicted by the perturbed parameters, and
\begin{equation*}
\delta\Vector{r}_k \Defined \bar{\Vector{r}}_k - \Vector{r}_k,
\quad
\delta\Vector{v}_k \Defined \bar{\Vector{v}}_k - \Vector{v}_k,
\quad
\delta\boldsymbol{\vartheta}_k \Defined \Matlog{\Matrix{C}_k^\T \bar{\Matrix{C}}_k}^{\vee},
\end{equation*}
with $\Matrix{C}_k$, $\Vector{v}_k$, and $\Vector{r}_k$ the navigation components of $\Matrix{Y}_k$.
The bias perturbations are included for completeness.
Their role in the symmetry-based analysis is discussed separately in \Cref{subsec:biases}.
The linearized measurement equation is
\begin{equation*}
\delta\Vector{y}_k \approx \Matrix{H}_k\<\delta\SmallStateVector,
\end{equation*}
where $\Matrix{H}_k \in \Real^{9 \times 16}$ is the measurement Jacobian, obtained by linearizing the delayed measurement model \Cref{eqn:SGal3_delayed_measurement} at the operating point.
The bias sensitivity matrices are
\begin{align*}
\Matrix{\Psi}^{(r)}_{\omega}(t)
& \Defined
\frac{\partial\<\breve{\Vector{r}}(t)}{\partial\<\Vector{b}_{\omega}},
\;\;
\Matrix{\Psi}^{(v)}_{\omega}(t)
\Defined
\frac{\partial\<\breve{\Vector{v}}(t)}{\partial\<\Vector{b}_{\omega}},
\;\;
\Matrix{\Psi}^{(\phi)}_{\omega}(t)
\Defined
\frac{\partial\<\delta\boldsymbol{\phi}_{\omega}(t)}{\partial\<\Vector{b}_{\omega}}, \\[1mm]
\Matrix{\Psi}^{(r)}_{\rms{a}}(t)
& \Defined
\frac{\partial\<\breve{\Vector{r}}(t)}{\partial\<\Vector{b}_{\rms{a}}},
\;\;
\Matrix{\Psi}^{(v)}_{\rms{a}}(t)
\Defined
\frac{\partial\<\breve{\Vector{v}}(t)}{\partial\<\Vector{b}_{\rms{a}}},
\end{align*}
where $\breve{\Matrix{C}}_{\omega}(t)$ is the orientation propagated with the perturbed gyroscope bias and $\delta\boldsymbol{\phi}_{\omega}(t) \Defined \smash{\ln\bigl(\breve{\Matrix{C}}(t)^\T\breve{\Matrix{C}}_{\omega}(t)\bigr)^{\vee}}$ is the right-local orientation perturbation.
These matrices depend on the input history and are not available in closed form in general.
The delay column contains the negated time derivatives $-\Vector{v}(t_{\rms{d},k})$, $-\Vector{a}(t_{\rms{d},k})$, and $-\Vector{\omega}(t_{\rms{d},k})$, where $\Vector{v}$ is the inertial-frame velocity, $\Vector{a}(t) \Defined \dot{\Vector{v}}(t) = \Matrix{C}(t)\<\Vector{s}(t) + \Vector{g}$ is the coordinate acceleration, and $\Vector{\omega}$ is the body-frame angular rate.
The full Jacobian is given in \Cref{eqn:single_full_jacobian} and derived in~\cite{2026_Kelly_Identifiability}.

For a single delayed full navigation-state measurement, the Jacobian $\Matrix{H}_k$ has at most nine independent rows, so it does not constrain all $16$ unknowns.
To isolate the delay--initial-condition ambiguity, we set $\delta\<\Vector{b}_{\omega} = \delta\<\Vector{b}_{\rms{a}} = \Vector{0}$ and examine the nullspace.
The biases are locally weakly observable along the trajectory~\cite{2011_Kelly_Visual} and are not part of the delay--initial-condition symmetry (see \Cref{subsec:biases}).
Defining
\begin{equation*}
\Vector{\gamma}
\Defined
\Matrix{C}_0\<\breve{\Vector{v}}(t_{\rms{d},k})^{\wedge}
\breve{\Matrix{C}}(t_{\rms{d},k})\<\Vector{\omega}(t_{\rms{d},k})
+
\Vector{a}(t_{\rms{d},k}),
\end{equation*}
the reduced Jacobian has a nullspace that contains, for any input history, the direction $\Vector{d}$, so that
\begin{equation}
\label{eqn:single_measurement_nullspace}
\Ker{\Matrix{H}_k}
\supseteq
\Span{\Vector{d}},
\end{equation}
with
\begin{equation*}
\Vector{d}
=
\bbm
\Vector{v}(t_{\rms{d},k})
-
t_{\rms{d},k}\<\Vector{\gamma}
+
\Matrix{C}_0\<\breve{\Vector{r}}(t_{\rms{d},k})^{\wedge}
\breve{\Matrix{C}}(t_{\rms{d},k})\<\Vector{\omega}(t_{\rms{d},k}) \\
\Vector{\gamma} \\
\breve{\Matrix{C}}(t_{\rms{d},k})\<\Vector{\omega}(t_{\rms{d},k}) \\
1 \\
\Zero_{3 \times 1} \\
\Zero_{3 \times 1}
\ebm.
\end{equation*}
This direction corresponds to a unit perturbation in the delay together with matching perturbations in the initial condition, leaving the measurement unchanged to first order.

Stacking measurements over time adds rows to the Jacobian.
Two delayed full-state measurements give an $18 \times 16$ stacked Jacobian that can have full column rank for generic input histories; hence, two measurements can be sufficient for local identifiability.
For pose-only (position and orientation) measurements, each measurement contributes six rows rather than nine, so two measurements are insufficient, while three can be sufficient generically~\cite{2023_Wise_Spatiotemporal}.

\section{Uninformative Trajectories}
\label{sec:uninformative_trajectories}

The main result of the paper is that there exists a family of trajectories for which the delay and initial condition cannot be uniquely recovered, regardless of the number or type of aiding measurements.
The key insight is geometric: informativeness is determined by the \emph{shape} of the trajectory, as encoded by the state transition matrix $\Matrix{\Phi}(t)$, since the gravity factor $\Matrix{G}(t)$ is common to every parameter pair.

\subsection{Uninformativeness}
\label{subsec:uninformativeness}

Any trajectory generated by a constant Lie algebra element is uninformative.
Along such a trajectory, a change in the delay can be compensated exactly by a shift of the initial condition along the trajectory.

\begin{theorem}[Constant-generator trajectories are uninformative]
\label{thm:constant_generator_uninformative}
Suppose the bias-corrected body-frame angular velocity and specific force are constant:
\begin{equation*}
\Vector{\omega}(t) = \Vector{\omega}_0,
\quad
\Vector{s}(t) = \Vector{s}_0,
\end{equation*}
with $\Vector{\omega}_0, \Vector{s}_0 \in \Real^3$.
Then the $\LieGroupSGal{3}$ state transition is generated by a fixed Lie algebra element,
\begin{equation}
\label{eqn:one_param_subgroup}
\Matrix{\Phi}(t)
=
\Matexp{t\<\Vector{\xi}_{\rms{c}}^{\wedge}},
\quad
\Vector{\xi}_{\rms{c}}^{\wedge}
=
\bbm
\Vector{\omega}_0^{\wedge} & \Vector{s}_0 & \Zero \\
\Zero & 0 & 1 \\
\Zero & 0 & 0
\ebm,
\end{equation}
where $\Matexp{t\<\Vector{\xi}_{\rms{c}}^{\wedge}}$ is defined for all $t \in \Real$, so that $\Matrix{\Phi}$ extends to negative arguments.
Let $\epsilon > 0$ be small enough that $\tau + \alpha \in (0, \tau_{\rms{u}})$ for all $\alpha \in (-\epsilon,\epsilon)$.
Define
\begin{equation}
\label{eqn:symmetry_map}
\bar{\tau}
\Defined
\tau + \alpha,
\quad
\bar{\Matrix{X}}_0
\Defined
\Matrix{T}(-\alpha)\<
\Matrix{G}(\alpha)\<
\Matrix{X}_0\<
\Matrix{\Phi}(\alpha).
\end{equation}
Then the transformation
\begin{equation*}
\mathcal{S}_{\alpha}:
\bigl(\Matrix{X}_0, \tau\bigr)
\mapsto
\bigl(\bar{\Matrix{X}}_0, \bar{\tau}\bigr)
\end{equation*}
with all other parameters left unchanged, is a continuous symmetry of the delayed measurement model.
Therefore, the delay--initial-condition pair is not jointly locally identifiable.
\end{theorem}

\begin{proof}
For the original parameters, the delayed full-state measurement at time $t$ is
\begin{equation*}
\Matrix{Y}(t)
=
\Matrix{T}(\tau)\<
\Matrix{G}(t_{\rms{d}})\<
\Matrix{X}_0\<
\Matrix{\Phi}(t_{\rms{d}}),
\end{equation*}
where $t_{\rms{d}} \Defined t - \tau$, and we write $\bar{t}_{\rms{d}} \Defined t - \bar{\tau} = t_{\rms{d}} - \alpha$ for the transformed delayed time.
For any $\beta$, direct computation yields the composition identity
\begin{equation}
\label{eqn:gravity_closure}
\Matrix{T}(\alpha)\<\Matrix{G}(\beta)\<\Matrix{T}(-\alpha)\<\Matrix{G}(\alpha)
=
\Matrix{G}(\alpha + \beta).
\end{equation}
For the transformed parameters in \Cref{eqn:symmetry_map},
\begin{align*}
\bar{\Matrix{Y}}(t)
& =
\Matrix{T}(\bar{\tau})\<
\Matrix{G}(\bar{t}_{\rms{d}})\<
\bar{\Matrix{X}}_0\<
\Matrix{\Phi}(\bar{t}_{\rms{d}}) \\
& =
\Matrix{T}(\tau)
\bigl[\Matrix{T}(\alpha)\<\Matrix{G}(\bar{t}_{\rms{d}})\<\Matrix{T}(-\alpha)\<\Matrix{G}(\alpha)\bigr]
\Matrix{X}_0\<
\Matrix{\Phi}(\alpha)\<
\Matrix{\Phi}(\bar{t}_{\rms{d}}) \\
& =
\Matrix{T}(\tau)\<
\Matrix{G}(t_{\rms{d}})\<
\Matrix{X}_0\<
\Matrix{\Phi}(t_{\rms{d}})
=
\Matrix{Y}(t),
\end{align*}
where the last step uses \Cref{eqn:gravity_closure} with $\beta = \bar{t}_{\rms{d}}$, the one-parameter subgroup property of $\Matrix{\Phi}(\cdot)$, and $\bar{t}_{\rms{d}} + \alpha = t_{\rms{d}}$.
Hence, the transformed parameters produce exactly the same delayed measurement history as the original parameters.

The transformed initial condition remains admissible: $\Matrix{G}(\alpha)$ is isochronous, and the time components of $\Matrix{\Phi}(\alpha)$ and $\Matrix{T}(-\alpha)$ cancel, so $\bar{\Matrix{X}}_0$ is isochronous whenever $\Matrix{X}_0$ is.
Since $\bar{\tau} \neq \tau$ for $\alpha \neq 0$, the result follows from \Cref{thm:symmetry_unidentifiable_continuous}.
\end{proof}

One application of \Cref{eqn:gravity_closure}, together with the one-parameter subgroup property of $\Matrix{T}(\cdot)$ and $\Matrix{\Phi}(\cdot)$, shows that the maps compose as $\mathcal{S}_{\alpha} \circ \mathcal{S}_{\beta} = \mathcal{S}_{\alpha+\beta}$, matching the group structure described in \Cref{subsec:symmetries}.

The next result gives a corresponding local necessity condition: a continuous delay-shift symmetry of the full delayed state history requires the generator to be constant over the delayed observation interval.
The only assumption on the family itself is that it traces a continuous curve of admissible delay--initial-condition pairs.

\begin{theorem}[Necessity for a continuous delay-shift symmetry]
\label{thm:symmetry_necessity}
With the IMU biases and all other parameters fixed, suppose the admissible pairs $\bigl(\bar{\Matrix{X}}_0, \tau + \alpha\bigr)$ depend continuously on $\alpha \in (-\epsilon, \epsilon)$, with $\bar{\Matrix{X}}_0 = \Matrix{X}_0$ at $\alpha = 0$, and leave the delayed state history unchanged.
Then the bias-corrected inputs are constant over the delayed observation interval, and
\begin{equation*}
\bar{\Matrix{X}}_0
=
\Matrix{T}(-\alpha)\<
\Matrix{G}(\alpha)\<
\Matrix{X}_0\<
\Matrix{\Phi}(t_{\rms{d}})\<
\Matrix{\Phi}(t_{\rms{d}} - \alpha)^{-1}
\end{equation*}
for any $t_{\rms{d}}$ in that interval.
\end{theorem}

\begin{proof}
Invariance of the delayed state history requires, for every $t$ in the observation interval and all $\alpha$ sufficiently small,
\begin{equation*}
\Matrix{T}(\bar{\tau})\<
\Matrix{G}(\bar{t}_{\rms{d}})\<
\bar{\Matrix{X}}_0\<
\Matrix{\Phi}(\bar{t}_{\rms{d}})
=
\Matrix{T}(\tau)\<
\Matrix{G}(t_{\rms{d}})\<
\Matrix{X}_0\<
\Matrix{\Phi}(t_{\rms{d}}).
\end{equation*}
Every factor is invertible, so
\begin{equation*}
\bar{\Matrix{X}}_0
=
\Matrix{G}(\bar{t}_{\rms{d}})^{-1}\<
\Matrix{T}(-\alpha)\<
\Matrix{G}(t_{\rms{d}})\<
\Matrix{X}_0\<
\Matrix{\Phi}(t_{\rms{d}})\<
\Matrix{\Phi}(\bar{t}_{\rms{d}})^{-1}.
\end{equation*}
The identity \Cref{eqn:gravity_closure} rearranges to
$\Matrix{G}(\bar{t}_{\rms{d}})^{-1}\<\Matrix{T}(-\alpha)\<\Matrix{G}(t_{\rms{d}})
= \Matrix{T}(-\alpha)\<\Matrix{G}(\alpha)$, and hence
\begin{equation}
\label{eqn:necessity_condition}
\bar{\Matrix{X}}_0
=
\Matrix{T}(-\alpha)\<
\Matrix{G}(\alpha)\<
\Matrix{X}_0\<
\Matrix{\Phi}(t_{\rms{d}})\<
\Matrix{\Phi}(\bar{t}_{\rms{d}})^{-1}.
\end{equation}
The left-hand side of \Cref{eqn:necessity_condition} depends only on $\alpha$, as do the three leading factors on the right, so the product $\Matrix{\Phi}(t_{\rms{d}})\<\Matrix{\Phi}(\bar{t}_{\rms{d}})^{-1}$ is independent of $t_{\rms{d}}$ on the delayed observation interval.
With $\alpha$ fixed, the derivative of this product with respect to $t_{\rms{d}}$ is zero.
Using the kinematics $\dot{\Matrix{\Phi}}(t) = \Matrix{\Phi}(t)\<\Vector{\xi}(t)^{\wedge}$ together with $\tfrac{d}{dt}\Matrix{\Phi}(t)^{-1} = -\Vector{\xi}(t)^{\wedge}\<\Matrix{\Phi}(t)^{-1}$, this gives
\begin{equation*}
\Matrix{\Phi}(t_{\rms{d}})\<\Vector{\xi}(t_{\rms{d}})^{\wedge}\<\Matrix{\Phi}(\bar{t}_{\rms{d}})^{-1}
-
\Matrix{\Phi}(t_{\rms{d}})\<\Vector{\xi}(\bar{t}_{\rms{d}})^{\wedge}\<\Matrix{\Phi}(\bar{t}_{\rms{d}})^{-1}
=
\Zero.
\end{equation*}
The outer factors are invertible, so $\Vector{\xi}(t_{\rms{d}}) = \Vector{\xi}(\bar{t}_{\rms{d}})$ for every $t_{\rms{d}}$ in the interval and every $\alpha \in (-\epsilon, \epsilon)$.
Invariance under all sufficiently small shifts forces $\Vector{\xi}$ to be constant on the interval.
Writing $\Vector{\xi}(t) = \Vector{\xi}_{\rms{c}}$ for the constant value, the blocks of $\Vector{\xi}_{\rms{c}}$ in \Cref{eqn:one_param_subgroup} give $\Vector{\omega}(t) = \Vector{\omega}_0$ and $\Vector{s}(t) = \Vector{s}_0$, so the bias-corrected inputs are constant over the delayed observation interval.
Since $\Matrix{\Phi}(t_{\rms{d}})\<\Matrix{\Phi}(\bar{t}_{\rms{d}})^{-1}$ is independent of $t_{\rms{d}}$, \Cref{eqn:necessity_condition} holds with the product evaluated at any $t_{\rms{d}}$ in the delayed observation interval, which is the stated form of $\bar{\Matrix{X}}_0$.
\end{proof}

If the inputs are constant from the initial time $t = 0$, then $\Matrix{\Phi}(t_{\rms{d}})\<\Matrix{\Phi}(t_{\rms{d}} - \alpha)^{-1} = \Matrix{\Phi}(\alpha)$ and the family reduces to the symmetry in \Cref{eqn:symmetry_map}.
\Cref{thm:symmetry_necessity} applies to the full delayed state history; sufficiency extends to partial aiding measurements, whereas necessity is established only for the full-state case.

\subsection{Symmetry Orbits}
\label{subsec:symmetry_orbits}

The symmetry in \Cref{thm:constant_generator_uninformative} defines a one-dimensional orbit in the space of delay--initial-condition pairs.
For a nominal pair $(\Matrix{X}_0, \tau)$, this orbit is
\begin{equation*}
\begin{aligned}
\Orb{
\vphantom{\begin{aligned} \\ \end{aligned}}
(\Matrix{X}_0, \tau)}
=
\Bigl\{
&
\bigl(
\Matrix{T}(-\alpha)\<
\Matrix{G}(\alpha)\<
\Matrix{X}_0\<
\Matrix{\Phi}(\alpha),
\;
\tau + \alpha
\bigr)
\\[-1mm]
& \hspace{26.3mm}
:
\alpha \in (-\epsilon, \epsilon)
\Bigr\}.
\end{aligned}
\end{equation*}
Every pair on the orbit produces the same delayed measurement history.
The maps form a local one-parameter group, $\mathcal{S}_{\alpha} \circ \mathcal{S}_{\beta} = \mathcal{S}_{\alpha+\beta}$, by \Cref{eqn:gravity_closure} and the subgroup property of $\Matrix{T}(\cdot)$ and $\Matrix{\Phi}(\cdot)$, so the orbit is a one-dimensional subset of the equivalence class of mutually indistinguishable parameter pairs.
For full-state measurements with the biases held fixed, \Cref{thm:symmetry_necessity} shows that any continuous family of indistinguishable pairs lies on this orbit.
For partial measurements, the orbit remains a subset of the equivalence class.
The transformed initial condition has the same navigation components as the state obtained by extending the constant-generator trajectory to time $\alpha$, while $\Matrix{T}(-\alpha)$ resets the group time coordinate to zero (see \Cref{fig:trajectory_teaser}).

Differentiating the transformed initial condition of \Cref{eqn:symmetry_map} at $\alpha = 0$ gives the orbit tangent.
Writing the navigation state components as
$\bar{\Vector{r}}_0(\alpha)$, $\bar{\Vector{v}}_0(\alpha)$, and
$\bar{\Matrix{C}}_0(\alpha)$, and using the right-local orientation coordinate
\begin{equation*}
\delta\boldsymbol{\phi}_0(\alpha)
\Defined
\Matlog{
\Matrix{C}_0^\T\<\bar{\Matrix{C}}_0(\alpha)
}^{\vee},
\end{equation*}
the tangent is
\begin{equation}
\label{eqn:orbit_tangent}
\left.\frac{d}{d\alpha}
\bbm
\bar{\Vector{r}}_0(\alpha) \\
\bar{\Vector{v}}_0(\alpha) \\
\delta\boldsymbol{\phi}_0(\alpha) \\
\tau+\alpha
\ebm\right|_{\alpha=0}
=
\bbm
\Vector{v}_0 \\
\Matrix{C}_0\<\Vector{s}_0 + \Vector{g} \\
\Vector{\omega}_0 \\
1
\ebm.
\end{equation}
The velocity row is the coordinate acceleration $\smash{\Vector{a}(0)} = \smash{\Matrix{C}_0\<\Vector{s}_0 + \Vector{g}}$, so the tangent is the time derivative of the navigation state and delay along the trajectory, as expected for a time-shift symmetry.
Comparing with \Cref{eqn:single_measurement_nullspace}, this vector agrees with the Jacobian null direction at every measurement time when the bias perturbations are set to zero.
The Jacobian nullspace is therefore the local, first-order `signature' of the same symmetry, and the stacked Jacobian is rank deficient because the same direction lies in the kernel of every block.

\begin{figure*}
\centering
\includegraphics[width=0.323\textwidth,trim={80px, 200px, 80px, 200px},clip]%
{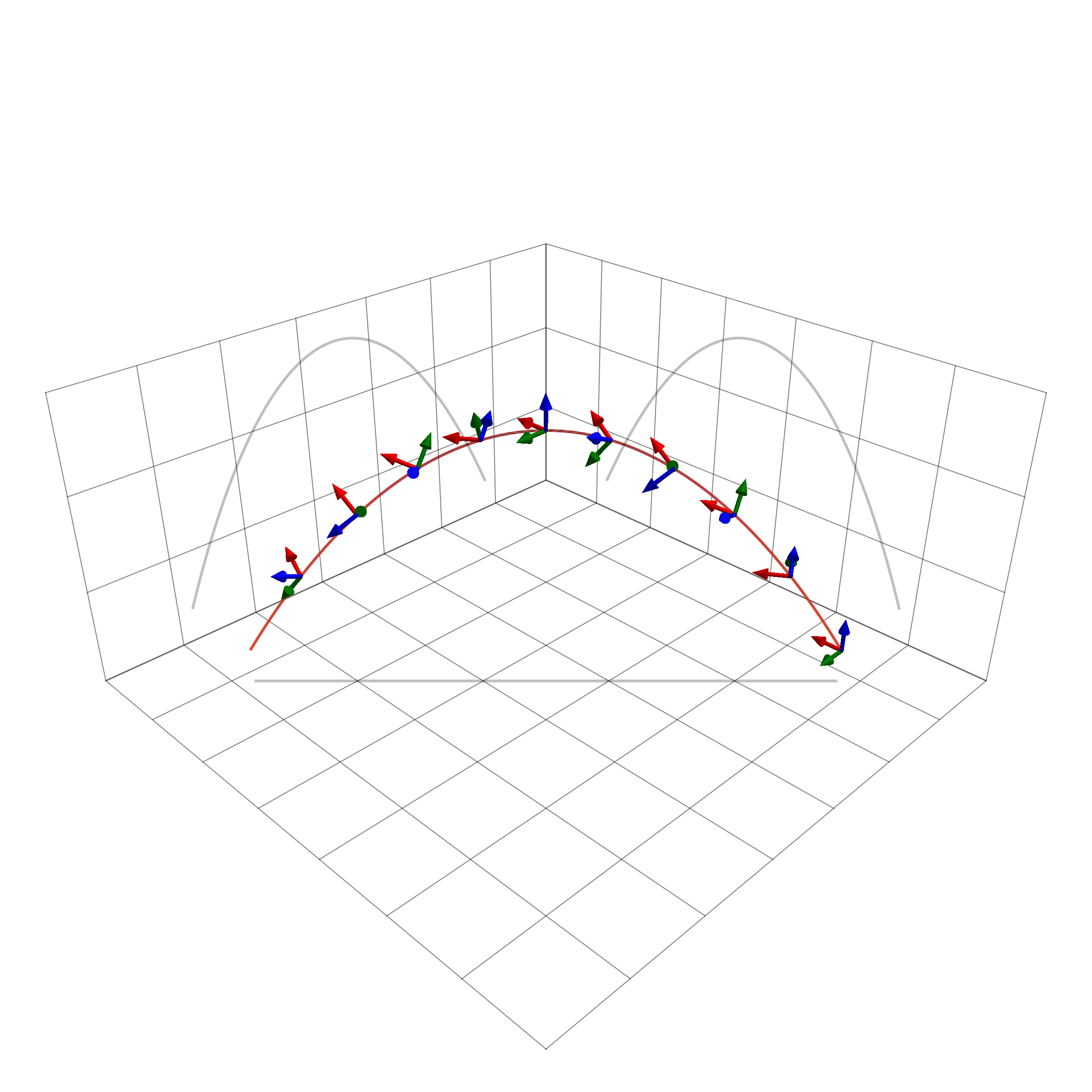}
\hspace{0.5mm}
\includegraphics[width=0.323\textwidth,trim={80px, 200px, 80px, 200px},clip]%
{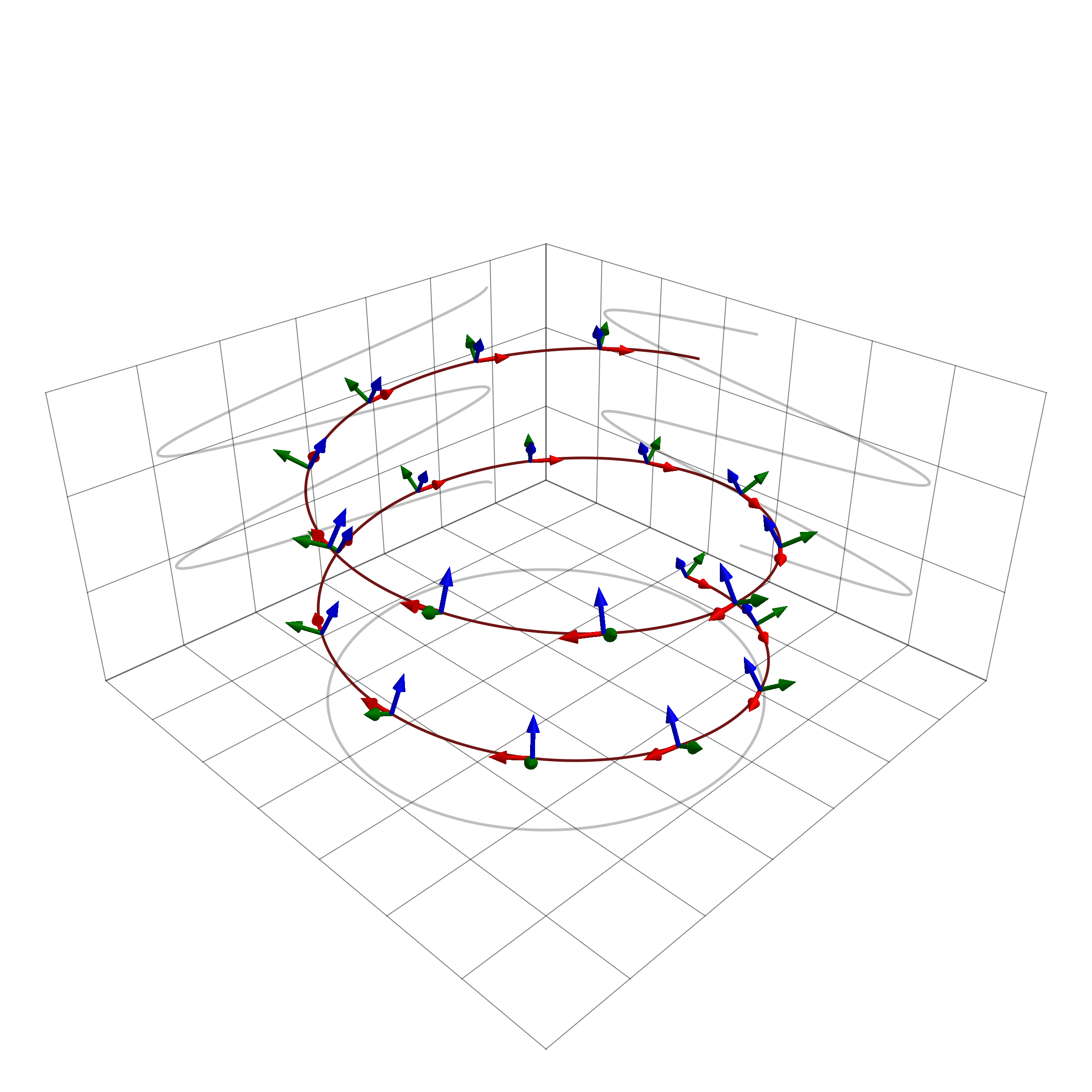}
\hspace{0.5mm}
\includegraphics[width=0.323\textwidth,trim={80px, 200px, 80px, 200px},clip]%
{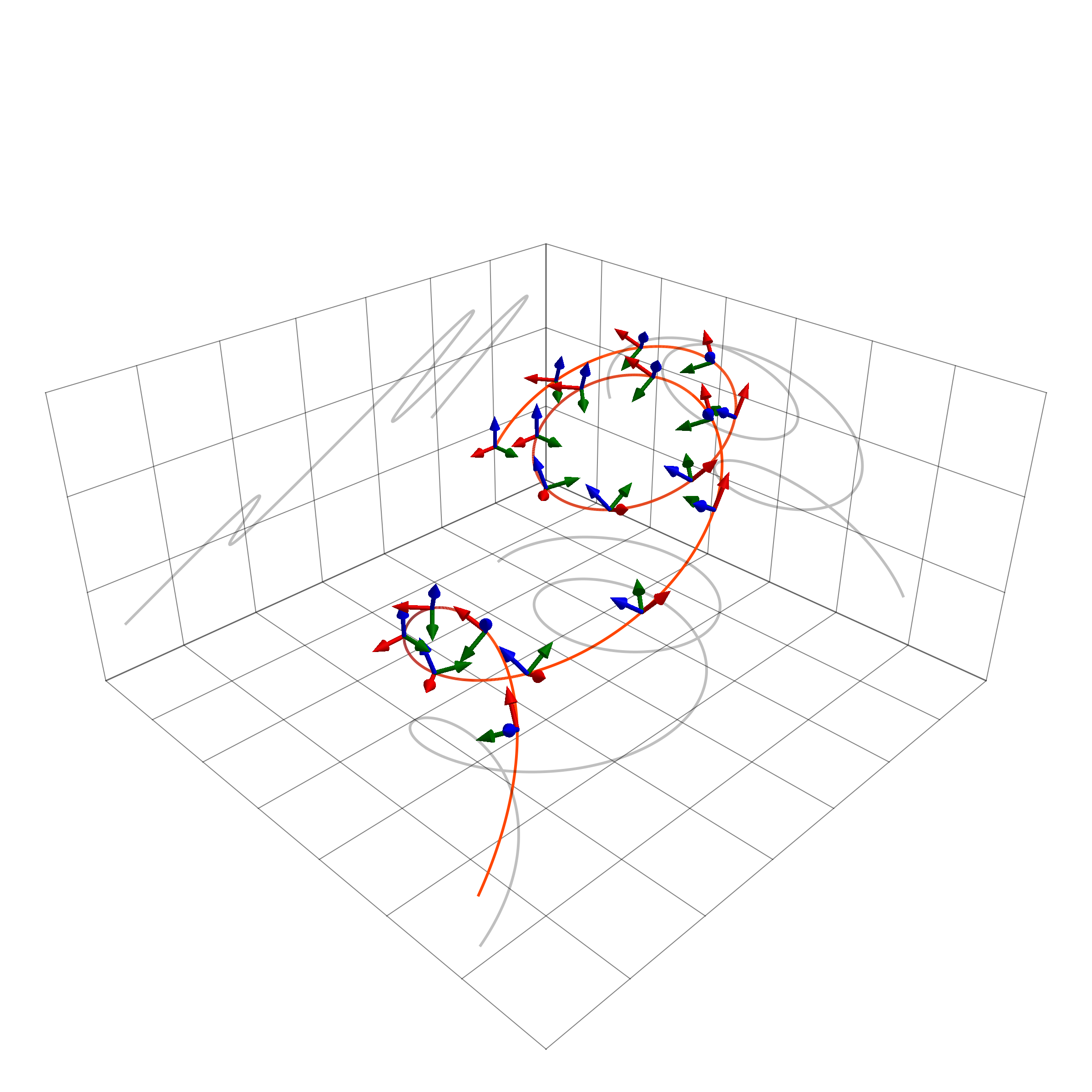}%
\vspace*{-1mm}
\caption{Examples of subfamilies of uninformative trajectories generated by constant body-frame inputs under gravity.
\emph{Left}: ballistic motion, produced by a zero specific force reading together with a constant spin.
\emph{Centre}: a climbing coordinated turn, with the bank angle fixed by the turn geometry, circular horizontal motion, and constant vertical velocity in the inertial frame.
\emph{Right}: general curvilinear motion, with the rotation axis tilted away from the gravity direction.
Body-frame axes are plotted along each trajectory, with the $x$, $y$, and $z$ axes in red, green, and blue.
All three trajectories share the same scale.
Each curve is shaded by the magnitude of the inertial-frame velocity, while the grey projections onto the coordinate planes show the trajectory shape.}
\label{fig:uninformative_examples}
\vspace*{-5mm}
\end{figure*}

\subsection{Gyroscope and Accelerometer Biases}
\label{subsec:biases}

The bias terms do not enter the symmetry in \Cref{thm:constant_generator_uninformative}.
Our prior work showed that the biases are locally weakly observable at each point along the trajectory~\cite{2011_Kelly_Visual}.
Perturbations to the biases change the bias-corrected body-frame input history and therefore the shape of the trajectory locally.

\subsection{Practical Examples}
\label{subsec:practical_examples}

From the form of \Cref{eqn:one_param_subgroup}, we can recognize several subfamilies of uninformative motion.
The subfamilies below are distinguished by the form of the generator $\Vector{\xi}_{\rms{c}}$:
\begin{itemize}
\setlength{\itemsep}{2pt}
\item \emph{Ballistic motion}, produced by zero specific force, $\Vector{s}_0 = \Zero$, and constant (possibly zero) vehicle-frame angular velocity.
The vehicle is in free fall, spinning at a constant rate, and the trajectory is a parabola determined by gravity and the initial velocity.
\item \emph{Constant coordinate acceleration}, produced by a nonzero specific force aligned with the spin axis, or with zero spin, so that $\Vector{\omega}_0^{\wedge}\Vector{s}_0 = \Zero$.
The inertial-frame acceleration is constant and the path is a parabolic arc.
\item \emph{Circular and helical motion}, produced by a specific force with a component perpendicular to the spin axis, with the axis parallel to gravity.
The motion is circular when the velocity is perpendicular to the spin axis, and helical otherwise.
A coordinated turn at constant bank angle is an example.
\item \emph{General curvilinear motion}, produced by any other constant pair $(\Vector{\omega}_0, \Vector{s}_0)$.
The trajectory shown in \Cref{fig:trajectory_teaser} is an example.
\end{itemize}

Representative trajectory examples are shown in \Cref{fig:uninformative_examples}.
Taken together, they show that the constant-generator condition is not a narrow special case, but a substantial geometric family that includes practically relevant motions.
The symmetry-based analysis also explains \emph{why} these trajectories are uninformative: the delayed measurement model admits a continuous symmetry, as established by \Cref{thm:constant_generator_uninformative}.

\section{Conclusion}
\label{sec:conclusion}

In this paper, we studied state and delay identifiability in aided inertial navigation with delayed measurements, using the special Galilean group $\LieGroupSGal{3}$.
Our formulation explicitly links the aided-navigation kinematics, delay, initial condition, and measurement history.
We described a continuous symmetry of the delayed measurement model under Galilean transformations and connected the symmetry to the nullspace of the measurement Jacobian.

Our main result is a geometric characterization of uninformative trajectories.
Identifiability depends not only on the number and type of measurements, but also on the shape of the vehicle trajectory itself.
When the bias-corrected inputs are constant, the state transition matrix is a one-parameter subgroup of $\LieGroupSGal{3}$.
A change in the delay can then be compensated by a shift of the initial condition along the trajectory, such that the delayed measurement history is unchanged.
The resulting symmetry orbit contains indistinguishable parameter tuples, so the initial condition and delay cannot be uniquely recovered.

There are several directions for future work.
One is to extend the analysis to relative measurements, including SLAM-type problems in which an additional spatial transformation appears in the measurement model~\cite{2011_Martinelli_State}.
Another is full spatiotemporal calibration, where unknown spatial extrinsic parameters and delays are estimated together~\cite{2019_Yang_Degenerate,2023_Yang_Online}.
The same ideas apply beyond inertial navigation, to other systems evolving on Lie groups with delayed sensor measurements.

\section*{Acknowledgements}

We thank the anonymous reviewers for their valuable comments and suggestions, which helped us to improve the quality and clarity of this manuscript.

\printbibliography

\end{document}